\documentclass[runningheads,envcountsame]{llncs}

\usepackage{graphicx}
\usepackage[utf8]{inputenc}
\usepackage[algo2e, ruled, vlined, linesnumbered]{algorithm2e}
\usepackage{amssymb, amsmath}
\usepackage{enumitem}
\usepackage{cite}
\usepackage{placeins}

\usepackage{ntheorem}
\newtheorem*{theorem*}{Theorem}
\newtheorem*{lemma*}{Lemma}
\newtheorem*{corollary*}{Corollary}

\providecommand{\ignore}[1]{}

\renewcommand{\epsilon}{\varepsilon}

\newcommand*{\natnum}{\mathbb{N}}
\newcommand*{\realnum}{\mathbb{R}}
\newcommand*{\oneone}{${(1{+}1)}$~\textup{EA}\xspace}
\newcommand*{\OM}{\textup{\textsc{One\-Max}}\xspace}
\newcommand*{\LO}{\textup{\textsc{Leading\-Ones}}\xspace}

\newcommand*{\fork}{\textup{\textsc{Fork}}\xspace}

\newcommand*{\x}{\mathbf x}

\newcommand*{\y}{\mathbf y}
\newcommand*{\Or}{\mathrm O}

\newcommand*{\Fork}[2]{\textup{\textsc{$\text{Fork}^\MakeLowercase{#1}_\MakeLowercase{#2}$}}}

\begin{document}
\title{Ring Migration Topology\\ Helps Bypassing Local Optima}
%
%
\author{Clemens Frahnow\and Timo K\"otzing}
\authorrunning{C. Frahnow and T. K\"otzing}
%
\institute{Hasso Plattner Institute, Prof.-Dr.-Helmert-Straße 2-3, 14482 Potsdam, Germany
\email{hpi-info@hpi.de}\\
\url{https://hpi.de/friedrich}}
\maketitle              
\begin{abstract}
	Running several evolutionary algorithms in parallel and occasionally exchanging good solutions is referred to as island models. 
	The idea is that the independence of the different islands leads to diversity, thus possibly exploring the search space better. Many theoretical analyses so far have found a complete (or sufficiently quickly expanding) topology as underlying migration graph most efficient for optimization, even though a quick dissemination of individuals leads to a loss of diversity.
	
	We suggest a simple fitness function \fork with two local optima parame\-trized by $r \geq 2$ and a scheme for composite fitness functions. 
	We show that, while the (1+1) EA gets stuck in a bad local optimum and incurs a run time of $\Theta(n^{2r})$ fitness evaluations on \fork, island models with a complete topology can achieve a run time of $\Theta(n^{1.5r})$ by making use of rare migrations in order to explore the search space more effectively. Finally, the ring topology, making use of rare migrations and a large diameter, can achieve a run time of $\tilde{\Theta}(n^r)$, the black box complexity of \fork.	This shows that the ring topology can be preferable over the complete topology in order to maintain diversity.

\end{abstract}

\section{Introduction}
\label{sec:intro}

In heuristic optimization, evolutionary algorithms are a technique that is capable of finding good solutions by employing strategies inspired by evolution \cite{Back1997}.
One way to understand why some optimization algorithms are more successful than others is to prove rigorous run time bounds on test functions which embody a typical challenge occurring in realistic optimization problems. For example, the famous \OM function, which assigns the number of $1$s of a bit string $x$ as the fitness of $x$, embodies the challenge of solving independent problems concurrently. The \LO function, counting the number of leading $1$s of a bit string, embodies the problem of solving otherwise independent problems sequentially (where the order is typically unknown). Thus, \OM and \LO are fruitful test functions to analyze search heuristics on, they simulate important properties of realistic search spaces in an analyzable way.

We introduce a new representative fitness function \fork which poses a choice of two possible directions, the fork. One of the directions is a dead end, a local and not global optimum; we call this the \emph{valley}. The other is the global optimum. We use a parameter $r > 1$ and formalize \fork by using \OM, but assigning two elements with (disjoint) sets of $r$ $0$s to be the global optimum and the valley. Thus, once trapped in the valley, it is hard to find the global optimum. Since the probability of finding the global optimum before the valley is exactly $1/2$ due to the symmetry of the search space (see also Lemma~\ref{lem:fork_half}), making random restarts with the well-known \oneone (or about $\log n$ independent runs) can efficiently find the global optimum.\\
\indent However, for realistic optimization problems, forks can happen not just as the last step of the search, but over and over again. The probability that a run will succeed and choose the right path each time decreases exponentially with the number of fork decisions to be made. We formalize this with \emph{composite fitness functions}. We give a general scheme for building fitness functions out of base functions by dividing the bit string into blocks. As an example for solving $k$ successive \fork functions, we can divide the bit string of length $n$ into $k$ equal parts. Each part contributes to the total fitness with its \fork-value, but only if all previous blocks are already optimized; this scheme was already used in essence by \cite{Lassig2013}. Intuitively, the resulting composite fitness function is like \LO, where each bit is a \fork function on bit strings of length $n/k$. Clearly, the \oneone as well as independent runs on such a succession of \fork functions are unlikely to succeed.\\
\indent Exactly to deal with such fitness landscapes, different techniques have been introduced.
One possible way in evolutionary computation is to employ \emph{island models}, meaning multiple computing agents (so called islands) that run the same algorithm in parallel and which can share information. Various analyses of island models have been made that show its usefulness:  Alba used parallel evolutionary algorithms to achieve super-linear speedups \cite{Alba2002}.
L\"assig and Sudholt gave a formal analysis of island models for many different migration topologies in \cite{LassigS14} showing how one can gain a speedup from parallelism; 
Badkobeh, Lehre and Sudholt could even show where the cut-off points are from which on linear speedup is no longer possible and discovered some bounds on different topologies \cite{Badkobeh2015}. In 2017, Lissovoi and Witt explored the performance of a parallel approach on dynamic optimization problems \cite{Lissovoi2017}.
In all these works, the idea is to exploit the computing power that comes along with multiple islands, gaining a speedup from parallelism.
Intuitively, a set up where each islands sends its best solution to all islands (called a \emph{complete migration topology}) as often as possible leads to the smallest run times, since all islands can share the progress of all others. Doerr et al.\ showed that if one considers the communication between islands also as time consuming, Rumor Spreading or Binary Trees perform even better on \OM and \LO \cite{Doerr2017}, but still the emphasis is on informing all islands as efficiently as possible about every improvement found.\\
\indent An essentially different work was given by L\"assig and Sudholt in \cite{Lassig2013}. They used a composite fitness function where each component tries to trick the algorithm to walk up a path leading to a local optimum, which is hard to escape (similar to \fork introduced above, but here we have paths that lead to the local optima). They give an island setting that can efficiently optimize this composite function, while simple hill climbers get stuck with high probability. Note that the complete topology also performs well in this setting, even though such high connectivity typically implies the loss of diversity, which was found important in many areas of heuristic optimization. Here the diversity was maintained by focusing on rare migration and making sure that migration only occurs at opportune times.\\
\indent In this paper we want to show that a high connectivity in a topology, for any frequency of migration, can lead to a loss of diversity and therefore to worse run times on \fork. In contrast to this we will show that the ring topology allows to maintain diversity.
We choose the ring on $\lambda$ vertices, since it is the unique graph with maximal diameter among all vertex transitive graphs with $\lambda$ vertices, in contrast to the complete graph, which has minimal diameter, thus highlighting the role of a large diameter (which implies a slow spread of migrants).\\
\indent First, in Section~\ref{sec:islands}, we introduce the algorithms we deal with. Section~\ref{sec:groups} introduces the fitness functions more formally, especially \fork and our scheme for composite fitness functions; here we also give a general result for the \oneone applied to such composite fitness functions which is of independent interest.\\
\indent In Section~\ref{sec:fork} we show that the \oneone fails to optimize \fork efficiently, with an expected run time of $\Theta(n^{2r})$ (see Theorem~\ref{theo:fork_general}). Independent runs of the \oneone similarly fail for compositions of \fork-functions.\\
\indent Regarding island models, while it is typical to consider as optimization time the time until just one island has found the optimum, we consider the time until \emph{all} islands have found the optimum: consider the case of optimizing $k$ successive \fork functions as introduced above. In order to be able to continue optimization after the first \fork function has been optimized, we need a sufficient number of islands which have passed this first phase; if we were to lose a constant fraction for each \fork function, then quickly all islands would be used up and the algorithm will get stuck in a local optimum. If \emph{all} islands make it to the next stage, then the optimization can proceed as in the first stage. We leave the rigorous argument to future work and contend ourselves with finding the time until all islands find the optimum of \fork, showing in what way the ring topology can be beneficial and, in fact, preferable to the complete topology.\\
\indent In Section~\ref{sec:complete} we consider an island model with the complete topology, that is, the different islands run a \oneone, but they occasionally share their best individual with \emph{all} other islands. In particular, we use a parameter $\tau$ such that each round with probability $1/\tau$ each island sends its best (and only current) individual to all other islands, continuing the search with the best individual among all incoming and own individuals.\footnote{Note that in some papers migration is considered to happen deterministically every $\tau$ rounds.}
For optimal choice of $\tau$ and the number of islands $\lambda$, the time until all islands have found the optimum here is $\Theta(n^{1.5r})$ fitness evaluations (see Corollary~\ref{cor:RunTimeComplete}).\\
\indent Next, in Section~\ref{sec:ring}, we show that the ring topology requires only $O(n^r (\log n)^2)$ fitness evaluations until all islands have found the optimum (see Corollary~\ref{cor:RunTimeRing}), which equals, up to polylogarithmic factors, the black-box complexity of \fork, which is $\Theta(n^r)$ (see Proposition~\ref{prop:BlackBoxComplexity}). In this sense the ring topology achieves the best possible optimization time over all black-box algorithms (up to the factor of $(\log n)^2$).\\
\indent Finally, in Section~\ref{sec:conclusion}, we conclude the paper with some final remarks.
Almost all proofs are omitted due to space constraints, but they can be found in the appendix.
%
%

\section{Algorithms}
\label{sec:islands}

The island model makes use of the $(1{+}1)$~Evolutionary Algorithm (\oneone for brevity).
The goal of that algorithm is to maximize a given fitness function by trying different search points and remembering the one that gave the best result so far.
The fitness function is defined on bit strings of a specific length $n$.
The algorithm starts with a bit string chosen uniformly at random.
A new individual for the input is generated each step by taking the best known input and flipping every bit independently with a probability of $\frac{1}{n}$ (\emph{standard bit mutation}).
If the fitness function yields a value that is not smaller than the best known so far, it becomes the new best individual.
Algorithm~\ref{alg:EA} makes this more formal.
\vspace*{-.6cm}
\FloatBarrier
\begin{algorithm2e}
	\SetAlgoSkip{tinyskip}
	$t \gets 0$\;
	$\x \gets$ solution drawn u.a.r.\ from $\{0,1\}^n$\;
	\While{termination criterion not met}
	{
		$t \gets t+1$\;
		$\y \gets$ flip each bit of $\x$ independently w/ prob.\ $1/n$\;
		\lIf{$f(\y) \ge f(\x)$}{
			$\x \gets \y$
		}
	}
	\caption{\oneone optimizing $f$.}
	\label{alg:EA}
\end{algorithm2e}
\vspace*{-.6cm}
Usually the termination criterion is met when the optimum is found.
In later chapters we let this algorithm run in parallel multiple times until the optimum is found everywhere.
In this case we change the termination criterion accordingly.
The run time of the \oneone is determined by the value of $t$ after the algorithm terminates.
For our research we use the \emph{island model} as an approach on parallel evolutionary computation, cf.~\cite{Rucinski2010,Neumann2011,LassigS14}.
The topology is defined by an undirected graph $G = {(V, E)}$, where $\lambda = |V|$.
Every vertex, called \emph{island}, represents an independent agent running the \oneone using standard bit mutation.
Like in the \oneone, the initial bit string is chosen uniformly at random.
This happens independently on every island.
All islands run in lockstep, meaning they all make the same amount of fitness evaluations in the same time.
Copies of the best found individual so far are shared along the edges of $G$ whenever a migration step happens.
An island overwrites its best solution when a received individual has a fitness that is not smaller than the resident best individual.
Ties among incoming migrants (with maximum fitness) are broken uniformly at random.
With a probability of $\frac{1}{\tau}$, every island sends its best individual to all of its neighbors.
Algorithm~\ref{alg:IslandModel} makes this more formal, where $\x^{(j)}$ denotes the best individual on island $j$.
Again, $t$ determines the run time (\emph{optimization time}) we are mainly interested in, as it counts the number of fitness evaluations of a single island.
If multiplied by $\lambda$, one gains the total number of fitness evaluations.
\vspace{-.6cm}
\FloatBarrier
\begin{algorithm2e}
\SetAlgoSkip{tinyskip}
    $t \gets 0$\;
    \For{$1 \le j \le \lambda$ \bf{in parallel}}
    {
        $\x^{(j)} \gets$ solution drawn u.a.r.\ from $\{0,1\}^n$\;
    }
    \While{termination criterion not met}
    {
    	$t \gets t+1$\;
    	$m \gets$ \textbf{true} with probability $1/\tau$ else \textbf{false}\;
        \For{$1 \le j \le \lambda$ \bf{in parallel}}
        {
            $\y^{(j)} \gets$ flip each bit of $\x^{(j)}$ independently w/ prob.\ $1/n$\;
            \lIf{$f(\y^{(j)}) \geq f(\x^{(j)})$}
            {
                $\x^{(j)} \gets \y^{(j)}$
            }
	        \If{m}
	        {
		        	Send $\x^{(j)}$ to all islands $k$ with $\{j,k\} \in E$\;\label{line:Communication}
		        	$N = \{\x^{(i)} \mid \{i,j\} \in E\}$\;
		        	$M = \{\x^{(i)} \in N \mid f(\x^{(i)}) = \max_{\x \in N} f(\x)\}$\;
		        	$\y^{(j)} \gets$ solution drawn u.a.r.\ from $M$\;
		        	\lIf{$f(\y^{(j)}) \geq f(\x^{(j)})$}
		            {
		                $\x^{(j)} \gets \y^{(j)}$
		            }
		        }
	        }
		}
\caption{Island model with migration topology $G = (V,E)$ on $\lambda$ islands and migration probability $1/\tau$.}
\label{alg:IslandModel}
\end{algorithm2e}
\vspace{-.6cm}
Observe that the final value of $t$ is a random variable;
in this paper whenever we use $T$ we refer to this random variable.

All bounds in this work will be in terms of $n$, $\lambda$, $\tau$ and $r$ simultaneously.
We consider $\lambda = \lambda(n)$ and $\tau = \tau(n)$ as positive, non-decreasing, integer-valued functions whereas $r$ is a fixed but arbitrary constant that has to be at least $2$.
The bounds we give describe the univariate asymptotics of the expected optimization times with respect to $n$ for any choices of $\lambda$ and $\tau$ within the given boundaries.

\section{Fitness Functions}
\label{sec:groups}

In this paper we investigate the maximization of pseudo-Boolean functions\linebreak $f \colon \{0,1\}^n \to \mathbb{R}_{\geq 0}$ on bit strings $\x = \x_0 \x_1 \dots \x_{n-1}$ of length $n$.
When we talk about the \emph{fitness} of a bit string $\x$ it refers to $f(\x)$.

We want to create composites of fitness functions by nesting them into each other.
To start, we will examine the run times of the two functions below.
Let $\left|\x\right|_1$ denote the number of bits set to $1$ within $\x$.
For $r\ge 2$ and $n\ge 2r$ we define

\begin{align*}
	&\LO(\x) = \sum_{i=0}^{n-1} \prod_{j=0}^i \x_j
	&\Fork{n}{r}(\x) = \begin{cases}
	n+1, &\text{if } \x = \{0\}^r\{1\}^{n-r};\\
	n+2, &\text{if } \x = \{1\}^{n-r}\{0\}^r;\\
	\left|\x\right|_1, &\text{otherwise.}
	\end{cases}
\end{align*}
In our work when we talk about the \Fork{n}{r} fitness function we will call the bit strings of fitness $n+1$ \emph{valley}, as getting to the optimum from there is harder than it is from any other fitness.
This definition of \fork is not related to the one in the work of Christian Gießen \cite{Giessen2013}.

We start by considering the general difficulty of optimizing \Fork{n}{r} as formalized by the \emph{unrestricted black-box complexity} \cite{Droste2006}, for which we consider the class of \Fork{n}{r} composed with any automorphism of the hypercube.
\begin{proposition}\label{prop:BlackBoxComplexity}
For constant $r$, the unrestricted black-box complexity of the \Fork{n}{r} function class is $\Theta(n^r)$.
\end{proposition}

\begin{proof}
	The lower bound comes from having to find at least one out of two search points among all search points with Hamming distance $r$ to the third-best individual, of which there are $\Theta(n^r)$ many. The upper bound comes from being able to find the third-best individual efficiently, for example with a \oneone, and then testing in a fixed order all search points with Hamming distance $r$ to this point.\qed
\end{proof}
\subsection{Composite Fitness Function}

Here we define composite fitness functions formally. 
Let $f = (f^n)_{n \in \natnum}$ be a family of fitness functions such that, for all $n \in \natnum$, $f^n: \{0,1\}^n \rightarrow \realnum_{\geq 0}$.
For this construction, we suppose that $1^n$ is the unique optimum (if a different bit string is the unique optimum, we apply a corresponding bit mask to make $1^n$ the unique optimum, but will not mention it any further).
With \emph{LeadingOnes with $k$-block $f$} we denote the fitness function which divides the input $\x$ into bit strings of length $k$ ($(\x_0 \x_1 \dots \x_{k-1}), (\x_k \x_{k+1} \dots \x_{2k-1}), \dots$). The blocks contribute to the total fitness using fitness function $f^k$, but only if all previous blocks have reached the unique optimum $1^k$. More formally,
\vspace*{-0.3cm}
\begin{align*}
	LO_k^f(\x) = \sum_{i=0}^{\frac{n}{k}-1} f^k(\x_{ik}\x_{ik+1} \dots\x_{ik+k-1}) \cdot \prod_{j=0}^{ik-1} \x_{j}.
\end{align*}
Similarly, \OM with $k$-block $f$ is defined as
\begin{align*}
	OM_k^f(\x) = \sum_{i=0}^{\frac{n}{k}-1} f^k(\x_{ik}\x_{ik+1} \dots\x_{ik+k-1}).
\end{align*}
In this paper we only analyze the run times around the \LO version.
Note that it equals the the $LOB_b$ function in the work of Jansen and Wiegand \cite{Jansen2003}.
In general, also other fitness functions can be used as the outer function of that nesting method, as exemplified in our definition of \OM with $k$-block.

\subsection{Run time of \LO with $k$-block $f$}
In the following we develop a general approach for proving run times on \LO with $k$-block $f$.
We make use of the observation that there is always exactly one block that contributes to the overall fitness except all previous blocks that are already optimal.

\begin{theorem}
	\label{theo:lo_exact}
	Let $T$ be the run time of a \oneone Algorithm on $LO_k^f$ to optimize a bit string of length $n$.	
	We have
	\begin{align*}
		E(T) = E\left(T^n_k\right) \frac{\left(\frac{n}{n-1}\right)^n-1}{\left(\frac{n}{n-1}\right)^k-1} \in \Theta\left(E\left(T^n_k\right)\frac{n}{k}\right)
	\end{align*}
	where $T^n_k$ is the run time to optimize $f^k$ with bit flip probability~$\frac{1}{n}$.
\end{theorem}

With this Theorem as a tool we can now derive exact bounds on functions like \LO.
The following equation was already shown by Sudholt \cite[Corollary 1]{Sudholt2010}. Here we use the approach of composite functions, which generalizes to many other fitness functions, but note that the underlying proof idea is the same.
\begin{corollary}\label{cor:lo_exact_bound}
	The expected run time of a \oneone on \LO is exactly
	\begin{align*}
		E(T_{LO}(n)) = \frac{\left(\frac{n}{n-1}\right)^{n-1}+\frac{1}{n}-1}{2}n^2.
	\end{align*}
\end{corollary}

\section{No Migration}
\label{sec:fork}

In the next sections we will frequently use that $m$ islands have to make a jump from the same fitness level to another one, where for each bit string in the current level the probability to do the jump is the same.
When we call $E_j$ the expected run time for a single island to do the jump, we can bound the expected run time $E(T)$ until one of them succeeds by
\begin{align}\label{eq:all_jump}
	\frac{E_j}{2m}\le E(T) \le \frac{E_j}{m} + 1.
\end{align}

This holds due to the fact that $E(T)$ is distributed geometrically with success probability $p=\frac{1}{E(T)}$ and because $\frac{pm}{1+pm} \le 1-\left(1-p\right)^m \le \frac{2pm}{1+pm}$ as shown by Badkobeh et al. \cite{Badkobeh2015}.
We already gave an intuition of the following lemma that will be used in several proofs.
\begin{lemma}\label{lem:fork_half}
	For any run of the \oneone on \Fork{n}{r} let $V$ denote the event that the valley string occurs as a best solution before the optimum.
	Then, $\Pr{(V)} = \frac{1}{2}$.
\end{lemma}

\subsection{The \oneone}

For the \oneone we can use Lemma~\ref{lem:fork_half} to see that it will be trapped with probability $1/2$, which leads to the following two theorems.

\begin{theorem}\label{theo:fork_general}
	The expected optimization time of the \oneone on \Fork{k}{r} with bit flip probability $\frac{1}{n}$ and $n\ge k$ is $E\left(T(k)\right) \in \Theta\left(n^{2r}\right)$.
\end{theorem}

\begin{theorem}\label{theo:single_fork_grouped}
	The expected optimization time of the \oneone on \LO with $k$-block \Fork{n}{r} is
		$E(T) \in \Theta\left(\frac{1}{k}n^{2r+1}\right)$.
\end{theorem}

\subsection{Independent Runs}
\label{sec:isolated}

In this section we examine the performance of $\lambda$ islands in isolation, all running the \oneone on $\Fork{n}{r}$ or \LO with $k$-block \Fork{n}{r}.
Isolation of the islands means there is no migration between them at all.
The next lemma will help us to get to the final bounds.
\begin{lemma}\label{lem:all_allone}
	Let $\lambda$ be the number of islands running the \oneone optimizing \Fork{n}{r}.
	Let $T_{all}^\lambda$ denote the run time for all islands to get to the optimum, valley or $1^n$ as their best solution respectively, regardless of the migration topology and policy.
	If $\lambda$ is polynomial in $n$, $E\left(T_{all}^\lambda\right) \in \Or\left(n\log n\right)$.
\end{lemma}
	
\begin{theorem}\label{theo:isolated_fork}
	For $\lambda \le n^r$ isolated islands the expected time to optimize \Fork{n}{r} is
	$E\left(T\right) \in \Or\left(n\log(n)+\frac{n^{2r}}{\lambda 2^\lambda}  + \frac{n^r}{\lambda}\right)$.
\end{theorem}

\begin{corollary}
	If we use $r\log n \le \lambda \le n^r$ islands, $E\left(T\right) \in \Or\left(n\log(n)+\frac{n^r}{\lambda}\right)$.
\end{corollary}

\begin{corollary}
	The expected number of evaluations until all $\lambda \le n^r$ islands get the optimum is in $\Omega\left(\lambda n^{2r}\log \lambda\right)$.
\end{corollary}
\begin{proof}
	To achieve this, all islands have to find the optimum.
	Observe that we expect half of the islands to get trapped.
	The probability that there are more than half as much is - by using Chernoff bounds - asymptotically more than a constant.
	Therefore we expect to need at least $\Omega\left(n^{2r}\log \lambda\right)$ rounds, which results in the given number of evaluations.\qed
\end{proof}

\begin{theorem}\label{theo:isolated_blocked_fork}
	For $k \le \frac{n}{\log\lambda}$, the expected run time of $\lambda$ islands optimizing \LO with $k$-block \Fork{n}{r} by running the \oneone can be bound by $E(T) \in \Omega\left(\frac{n^{2r}}{\lambda}\right)$.
\end{theorem}
\begin{proof}
	Let $V$ denote the event that every island gets trapped in a valley at least once during a run.
	From Lemma \ref{lem:fork_half} one can derive that $\Pr{(V)} = \left(1-\frac{1}{2^\frac{n}{k}}\right)^\lambda$.
	
	Again we apply the law of total expectation and use the lower bound $\frac{E}{2\lambda}$ on the expected run time of $\lambda$ islands until one of them makes a jump where $E$ is the expected number of steps for a single island (Equation \eqref{eq:all_jump}).
	\begin{align*}
		E(T) &= E\left(T \; \middle| \; V\right)\Pr{(V)} + E\left(T \; \middle| \; \overline{V}\right)\Pr{(\overline{V})}
		\;\ge\; \frac{n^{2r}}{2\lambda}\Pr{(V)}
		\;=\; \frac{n^{2r}}{2\lambda}\left(1-\frac{1}{2^\frac{n}{k}}\right)^\lambda
	\end{align*}
	Using that $k \le \frac{n}{\log(\lambda)}$ finally gives $E(T)\ge \frac{n^{2r}}{2\lambda}\left(1-\frac{1}{\lambda}\right)^\lambda
		\;\in\; \Omega\left(\frac{n^{2r}}{\lambda}\right)$.\qed
\end{proof}

\section{Complete Topology}
\label{sec:complete}

The disadvantage of the complete topology when optimizing \Fork{n}{r} is that if at least one island gets the chance to migrate the valley, all islands get trapped.
To get to the bounds, we first investigate the worst run time that could appear.
Second, we calculate a bound on the probability for one island to find the optimum or the valley during the time it takes all islands to get to $1^n$, the valley or the optimum.

In this and the following sections we frequently use the worst case run time that can occur if we do not find the optimum early enough.
\begin{lemma}\label{lem:worst_case_runtime}
	For $\lambda \in \Or\left(\frac{n^{2r-1}}{\log n}\right)$ islands being polynomial in $n$ and any migration topology and/or policy, the expected run time to optimize \Fork{n}{r} can be bounded by $E(T) \in \Or\left(\frac{n^{2r}}{\lambda}\right)$.
\end{lemma}

To get a lower bound on the run time we want to concentrate on the case that all islands come to a state where every island has $1^n$ as its solution.
That is why in the next lemma we show how likely it is that the optimum or the valley is generated before that state is reached.

\begin{lemma}\label{lem:ov_bound}
	The probability $p_{ov}$ for a single island to generate the optimum or the valley during its way to $1^n$ while optimizing \Fork{k}{r} with \oneone$_{1/n}$ is $p_{ov} \in \Or\left(\frac{1}{k^{r-1}}\right)\;\;\text{ for }k\le n$.
\end{lemma}

Next we give a lemma that we use for the upper and the lower bound, which considers the event that and islands finds the valley and broadcasts it, thus drowning diversity.
\begin{lemma}\label{lem:probability_Q}
	Under the condition of at least one of $\lambda \in \Or\left(n^{r-1}\right)$ island having $1^n$ and all others the valley as solution, the probability for the event $Q$ that one island will find the valley and a migration will be made before the optimum is found by any island is $\frac{c}{2} \frac{n^r}{n^r + \tau\lambda}
		\;\le\; \Pr{(Q)} \;\le\;
		\frac{n^r + \frac{\lambda}{2}}{2n^r + \tau\lambda}$
	for a constant $0<c<1$.
\end{lemma}

We continue with the lower bound for the complete topology, followed by the upper bound.
The restriction for $\lambda$ to be in $\Or\left(\frac{n^{r-1}}{\log n}\right)$ in the next theorems is useful since the expected number of derived optima during the first $\Or(n\log n)$ steps is constant if we chose $\lambda \in \Theta\left(n^{r-1}\right)$. This follows from Lemma \ref{lem:all_allone} and Lemma~\ref{lem:ov_bound}.

\begin{theorem}\label{theo:fork_comlete_lower_bound}
	The expected run time of $2r\log n \le \lambda \in \Or\left(\frac{n^{r-1}}{\log n}\right)$ islands optimizing \Fork{n}{r} on a complete graph is $E(T) \in \Omega\left(\frac{n^{3r} + \tau\lambda n^r}{\lambda n^r + \tau\lambda^2}\right)$.
\end{theorem}

\begin{theorem}\label{theo:fork_comlete_upper_bound}
	The expected run time of $2\log n \le \lambda \in \Or\left(\frac{n^{r-1}}{\log n}\right)$ islands optimizing \Fork{n}{r} on a complete graph is in
	$
		E(T) \in \Or\left(\frac{n^{3r} + \tau\lambda n^r}{\lambda n^r + \tau\lambda^2} + \frac{n^{2r+1}\log n}{\tau\lambda}\right)
	$.
\end{theorem}

\begin{corollary}\label{cor:RunTimeComplete}
	If we choose $\tau \in \Omega(n\log n)$ and $2r\log n \le \lambda \in \Or(n^{r-1-\epsilon})$ for $\epsilon > 0$ constant, then the optimization time for \Fork{n}{r} on a complete graph is $E(T) \in \Theta\left(\frac{n^{3r} + \tau\lambda n^r}{\lambda n^r + \tau\lambda^2}\right)$.
\end{corollary}
\begin{proof}
	We already have shown the lower bound in Theorem \ref{theo:fork_comlete_lower_bound}.
	The second term of the upper bound in Theorem \ref{theo:fork_comlete_upper_bound} is dominated by the rest if $\tau \in \Omega(n\log n)$.
	Therefore it matches the lower bound.\qed
\end{proof}

\begin{corollary}
	The number of fitness evaluations to spread the optimum to all islands is in $\Theta\left(n^{1.5r}\right)$ for the best choice of parameters $\lambda$ and $\tau$.
\end{corollary}
\begin{proof}
	This can be shown by recalling that the number of evaluations to get there is in $\Omega\left(\frac{n^{3r} + \tau\lambda n^r}{n^r + \tau\lambda} + \tau\lambda\right)$
	and in $\Or\left(\frac{n^{3r} + \tau\lambda n^r}{n^r + \tau\lambda} + \frac{n^{2r+1}\log n}{\tau} + \tau\lambda\right)$ (Theorem \ref{theo:fork_comlete_lower_bound} and \ref{theo:fork_comlete_upper_bound}).
	There is no way to choose $\tau\lambda$ to get below $\Theta\left(n^{1.5 r}\right)$.\qed
\end{proof}

\section{Ring Topology}
\label{sec:ring}

We expect the ring topology to perform better than the complete graph due to the fact that even if one island finds the valley, there is enough time for the others to come up with the optimum before they would get informed of the valley by a neighbor.
In this section we want to prove this assumption.

First we show that we expect just a small number of islands to find the valley before all other islands get to $1^n$.
After that we prove that we can choose a migration probability so that valleys are unlikely to be shared too often.
As final step we show that all other islands have enough time to find the optimum so that we get an upper bound of the expected run time.

For those steps we define two events that can occur.

\begin{definition}
	Let $V_b$ be the event that the valley was generated on at most $b$ islands during the time until all other islands have $1^n$, the valley or the optimum as their solution.
\end{definition}

\begin{definition}
	Let $B_c$ be the event that after the time until all islands have $1^n$, the valley or the optimum as their solution, there is a maximum of $c \log n$ valleys.
\end{definition}


\begin{lemma}\label{lem:ring_constant_valleys_until_allone}
	If $b\ge 1+\frac{r}{\epsilon}$ is constant and $\lambda \in \Or\left(n^{r-1-\epsilon}\right)$, where $\epsilon>0$ is a constant, it holds that the expected run time of to optimize \Fork{n}{r} on any island is
	$
		E(T) \le E{\left(T \;\middle|\; V_b\right)} + \Or\left(n\right)
	$.
\end{lemma}

Another event that could possibly lead to a high run time is when one of the found valleys is shared too fast to all other islands.
Like before we will show that this case is unlikely enough to not dominate the run time.

\begin{lemma}\label{lem:ring_constant_migration}
	Let $\epsilon>0$, $b\ge 1+\frac{r}{\epsilon}$ and $c\ge 7rb$ be constants and $\lambda \in \Or\left(n^{r-1-\epsilon}\right)$.
	When $\frac{1}{\tau}$ denotes the migration probability and $\tau \in \Omega\left(n\log n\right)$, the expected run time to optimize \Fork{n}{r} on any island is
	$
		E(T) \le E{\left(T \;\middle|\; V_b\cap B_c\right)} + \Or\left(\frac{n^r}{\lambda}\right).
	$
\end{lemma}

From the previous lemmas we can derive that if we choose $\tau$ large enough and $\lambda$ small enough, we get an upper bound on the expected run time by just looking at the case of $B_c\cap V_b$ and adding $\Or\left(\frac{n^r}{\lambda}\right)$.

\begin{theorem}\label{theo:ring_fork}
	For $\tau \in \Omega(n\log n)$, $12r\log n \le \lambda \in \Or\left(n^{r-1-\epsilon}\right)$ and $\lambda^2 \tau \ge 17r^2n^r\log^2 n$, the expected optimization time for \Fork{n}{r} on a Ring topology is
	$
		E(T) \in \Or\left(\frac{n^r}{\lambda}\right).
	$
\end{theorem}

The next corollary sums up our findings and shows performance for the optimal choice of parameters.

\begin{corollary}\label{cor:RunTimeRing}
	The expected number of fitness evaluations to have all islands at the optimum is in $\Or\left(n^r\log^2 n\right)$ if $\tau$ and $\lambda$ are set appropriately.
\end{corollary}
\begin{proof}
	If we use the results from Theorem \ref{theo:ring_fork}, we see that there are $\Or\left(n^r + \tau\lambda^2\right)$ evaluations until all islands are informed.
	If we consider that $\tau\lambda^2 \in \Or\left(n^r\log^2 n\right)$ we get the bound.\qed
\end{proof}

\section{Conclusion}
\label{sec:conclusion}

The results we obtained regarding the expected number of fitness evaluations of \fork by different algorithms (until all islands have the optimum, in case of an island model) and on optimal parameter settings are
\begin{itemize}
	\item \oneone -- $\Theta\left(n^{2r}\right)$;
	\item Independent runs -- $\Omega\left(\lambda n^{2r}\log \lambda\right)$;
	\item Complete -- $\Theta\left(n^{1.5r}\right)$;
	\item Ring -- $\Or\left(n^{r}\log^2 n\right)$.
\end{itemize}
To show this we exploited that less diversity can mean to be trapped, but it also gives advantages if there is migration between the islands. Note that a ring delays migration, since in every migration step an individual can only proceed by one island along the ring, so the total time to inform all islands is highly concentrated around the expectation. This is different in the complete topology with a high value of $\tau$: the expected time to inform all individuals might be the same, but the concentration is weaker.

We showed when a ring topology outperforms a topology with faster dissemination of individuals due to the increase in diversity. It would be interesting to see what other properties of the search space can also gain from a ring topology. Furthermore, one could wonder whether always one of the two extremes, complete and ring, is the best choice. There may be effects that can benefit for example a two-dimensional torus over both the ring and the complete graph.

We also discussed the method of composing different fitness functions, based on nesting one fitness function in another. We focused on the case that the outer fitness function is \LO and all inner fitness functions are \fork of the same length, but in principle one can consider inner fitness function that differ from each other, possibly also in their length and also other options for out fitness functions such as \OM. We gave a general but precise tool to calculate run times on \LO-composite fitness functions when using the \oneone.

\clearpage

\bibliographystyle{splncs04}
\bibliography{PPSN18_ring}

\begin{thebibliography}{10}
\providecommand{\url}[1]{\texttt{#1}}
\providecommand{\urlprefix}{URL }
\providecommand{\doi}[1]{https://doi.org/#1}

\bibitem{Alba2002}
Alba, E.: Parallel evolutionary algorithms can achieve super-linear
  performance. \relax Information Processing Letters  \textbf{82},  7--13
  (2002)

\bibitem{Back1997}
B{\"a}ck, T., Fogel, D.B., Michalewicz, Z.: Handbook of evolutionary
  computation. Release  \textbf{97}(1), ~B1 (1997)

\bibitem{Badkobeh2015}
Badkobeh, G., Lehre, P.K., Sudholt, D.: Black-box complexity of parallel search
  with distributed populations. In: Proceedings of the 2015 FOGA XIII. pp.
  3--15. ACM (2015)

\bibitem{Doerr2017}
Doerr, B., Fischbeck, P., Frahnow, C., Friedrich, T., K{\"o}tzing, T.,
  Schirneck, M.: Island models meet rumor spreading. In: Proceedings of the
  GECCO 2017. pp. 1359--1366. ACM (2017)

\bibitem{Doerr2013}
Doerr, B., Goldberg, L.A.: Adaptive drift analysis. Algorithmica
  \textbf{65}(1),  224--250 (2013)

\bibitem{Doerr2012}
Doerr, B., Johannsen, D., Winzen, C.: Multiplicative drift analysis.
  Algorithmica  \textbf{64},  673--697 (2012)

\bibitem{Droste2006}
Droste, S., Jansen, T., Wegener, I.: Upper and lower bounds for randomized
  search heuristics in black-box optimization. Theory of computing systems
  \textbf{39}(4),  525--544 (2006)

\bibitem{Eisenberg2008}
Eisenberg, B.: On the expectation of the maximum of iid geometric random
  variables. Statistics \& Probability Letters  \textbf{78}(2),  135--143
  (2008)

\bibitem{Giessen2013}
Gie{\ss}en, C.: Hybridizing evolutionary algorithms with opportunistic local
  search. In: Proceedings of the GECCO 2013. pp. 797--804. ACM (2013)

\bibitem{Jansen2003}
Jansen, T., Wiegand, R.P.: Exploring the explorative advantage of the
  cooperative coevolutionary (1+1) ea. In: Proceedings of the GECCO 2003. pp.
  310--321. Springer (2003)

\bibitem{Lassig2013}
L{\"a}ssig, J., Sudholt, D.: Design and analysis of migration in parallel
  evolutionary algorithms. Soft Computing  \textbf{17}(7),  1121--1144 (2013)

\bibitem{LassigS14}
L{\"{a}}ssig, J., Sudholt, D.: General upper bounds on the runtime of parallel
  evolutionary algorithms. Evolutionary Computation  \textbf{22},  405--437
  (2014)

\bibitem{Lissovoi2017}
Lissovoi, A., Witt, C.: A runtime analysis of parallel evolutionary algorithms
  in dynamic optimization. Algorithmica  \textbf{78}(2),  641--659 (2017)

\bibitem{Neumann2011}
Neumann, F., Oliveto, P.S., Rudolph, G., Sudholt, D.: On the effectiveness of
  crossover for migration in parallel evolutionary algorithms. In: Proceedings
  of the GECCO 2011. pp. 1587--1594 (2011)

\bibitem{Rucinski2010}
Ruci{\'n}ski, M., Izzo, D., Biscani, F.: On the impact of the migration
  topology on the island model. \relax Parallel Computing  \textbf{36},
  555--571 (2010)

\bibitem{Sudholt2010}
Sudholt, D.: General lower bounds for the running time of evolutionary
  algorithms. PPSN XI pp. 124--133 (2010)

\end{thebibliography}

\clearpage

\appendix
\section{Appendix}
\label{sec:appendix}

Here in the appendix we give all proofs omitted from the main part of the document.

\subsection{Proof of Theorem \ref{theo:lo_exact}}
\begin{theorem*}
	Let $T$ be the run time of a \oneone Algorithm on $LO_k^f$ to optimize a bit string of length $n$.
	
	It then holds that 
	\begin{align*}
		E(T) = E\left(T^n_k\right) \frac{\left(\frac{n}{n-1}\right)^n-1}{\left(\frac{n}{n-1}\right)^k-1} \in \Theta\left(E\left(T^n_k\right)\frac{n}{k}\right)
	\end{align*}
	
	where $T^n_k$ is the run time to optimize $f^k$ with a bit flip probability of $\frac{1}{n}$.
\end{theorem*}
\begin{proof}
	Let $T_i$ be the run time to optimize $f$ on the $i^{\text{th}}$ block $b_i$, given that all previous blocks are already optimized.
	By the definition of $LO_k^f$ we know that all blocks have to be optimized one after another from block $b_0$ to block $b_{\frac{n}{k}-1}$.
	Hence,
	\begin{align}\label{eq:gen_sum}
		E(T) = \sum_{i=0}^{\frac{n}{k}-1} E(T_i).
	\end{align}
	
	This still holds even if we optimize multiple blocks at once, because then the run time of the next blocks goes down to $0$.

	Let $p_i$ be the probability to not flip a bit in the first $i-1$ blocks.
	If a bit flips within the $i-1$ first blocks, the fitness no longer depends on the $i$th block that we want to optimize, hence whatever happened in the $i$th block will be discarded.
	Therefore if we only look at steps where no bit flip happened in those leading blocks, we get a run time of $T^n_k$.
	
	We want to prove this intuition by using Wald's equation.
	Let $X_j \in \{0,1\}$ be an indicator variable, that equals $1$ when no bit flip happened in the first $i$ blocks of the bit string we want to optimize, else $0$.
	By definition of \LO, all steps are discarded when a bit flip in these blocks were made.
	
	\begin{align*}
		&T^n_k = \sum_{j=1}^{T_i} X_j\\
		\intertext{Using Wald's Equation}
		&\Rightarrow E\left(T^n_k\right) = E(T_i) E(X_j) = E(T_i) p_i\\
		&\Rightarrow E(T_i) =  \frac{E\left(T^n_k\right)}{p_i} = \left(\frac{n}{n-1}\right)^{ik}E\left(T^n_k\right)
	\end{align*}
	Using that $i$ has a maximum value of $\frac{n}{k}-1$ we can also conclude,
	\begin{align}\label{eq:raw_bounds}
		E\left(T^n_k\right) \le E\left(T_i\right) = \left(\frac{n}{n-1}\right)^{ik}E\left(T^n_k\right) \le e E\left(T^n_k\right).
	\end{align}
	If we put that into Equation \eqref{eq:gen_sum} we get
	\begin{align*}
		E(T) &= \sum_{i=0}^{\frac{n}{k}-1} E(T_i)\\
		E(T)&= \sum_{i=0}^{\frac{n}{k}-1} \left(\frac{n}{n-1}\right)^{ik}E\left(T^n_k\right)\\
		&= E\left(T^n_k\right)\sum_{i=0}^{\frac{n}{k}-1} \left(\frac{n}{n-1}\right)^{ik}\\
		&= E\left(T^n_k\right) \frac{\left(\frac{n}{n-1}\right)^n-1}{\left(\frac{n}{n-1}\right)^k-1}.
	\end{align*}
	
	It directly follows from Equation \eqref{eq:gen_sum} and \eqref{eq:raw_bounds}, that
	
	\begin{align*}
		\frac{n}{k}E\left(T^n_k\right) \le E\left(T^n_k\right) \frac{\left(\frac{n}{n-1}\right)^n-1}{\left(\frac{n}{n-1}\right)^k-1} \le \frac{en}{k}E\left(T^n_k\right)
		\end{align*}
		and hence,
		\begin{align*}
		E(T) = E\left(T^n_k\right) \frac{\left(\frac{n}{n-1}\right)^n-1}{\left(\frac{n}{n-1}\right)^k-1} = \Theta\left(E\left(T^n_k\right)\frac{n}{k}\right).
	\end{align*}	
\end{proof}

\subsection{Proof of Corollary \ref{cor:lo_exact_bound}}
\begin{corollary*}
	The expected run time of a \oneone on \LO is exactly
	\begin{align*}
		E(T_{LO}(n)) = \frac{\left(\frac{n}{n-1}\right)^{n-1}+\frac{1}{n}-1}{2}n^2.
	\end{align*}
\end{corollary*}
\begin{proof}
	We use Theorem \ref{theo:lo_exact} and set $k=1$, so that $E\left(T^n_k\right)$ is just the expected run time to get a $1$ on one bit that is initialized randomly and has a flip probability $\frac{1}{n}$.
	In half of all cases we already start with a $1$ and are done in $0$ steps, otherwise we have a expected run time of $n$.
	This fitness function nested in \LO gives us \LO itself.
	Obviously, $E(T^n_k) = \frac{n}{2}$ which leads to
	\begin{align*}
		E(T_{LO}(n)) &= E\left(T^n_k\right) \frac{\left(\frac{n}{n-1}\right)^n-1}{\left(\frac{n}{n-1}\right)^k-1}\\
		&= \frac{n}{2} \frac{\left(\frac{n}{n-1}\right)^n-1}{\frac{n}{n-1}-1}\\
		&= \frac{\left(\frac{n}{n-1}\right)^{n-1}+\frac{1}{n}-1}{2}n^2.
	\end{align*}
\end{proof}

\subsection{Proof of Lemma \ref{lem:fork_half}}
\begin{lemma*}
	For any run of the \oneone on \Fork{n}{r} let $V$ denote the event that the valley string occurs as a best solution before the optimum.
	Then, $\Pr{(V)} = \frac{1}{2}$.
\end{lemma*}
\begin{proof}
	Let $p_S(V)$ be the probability of $V$ under the condition of starting with bit string $S$ and $S'$ be $S$ reversed.
	It then holds that $p_S(V) = p_{S'}(\overline{V})$ due to the fact that the fitnesses of reversed strings stay the same with the exception of the optimum and the valley and the valley reversed is the optimum.
	Since they are the last strings of a run, they do not influence the probabilities of all other strings to occur before.
	
	From the law of total probabilities we can conclude,
	\begin{align*}
		\Pr{(V)} &= \sum_{S\in \{0,1\}^n} p_S(V)\frac{1}{2^n}
		= \sum_{S'\in \{0,1\}^n} p_{S'}(V)\frac{1}{2^n}\\
		&= \sum_{S\in \{0,1\}^n} p_S(\overline{V})\frac{1}{2^n}
		= \Pr{(\overline{V})}.
	\end{align*}
	Hence,
	$\Pr{{V}} = \frac{1}{2}$.
\end{proof}

\subsection{Proof of Theorem \ref{theo:fork_general}}
\begin{theorem*}
	The expected optimization time of the \oneone on \Fork{k}{r} with bit flip probability $\frac{1}{n}$ and $n\ge k$ is
	\begin{align*}
	E\left(T(k)\right) \in \Theta\left(n^{2r}\right).
	\end{align*}
\end{theorem*}
\begin{proof}
	First we show $E\left(T(k)\right) \in \Or\left(n^{2r}\right)$ using the fitness level argument.
	Assuming the worst case of getting into every possible fitness level we get an upper bound on the expected run time by adding up the expected run times of leaving all the levels for a better fitness.
	
	Let level $L_i$ be the set of all bit strings that yield the same fitness $i$, so that $\forall s_i \in L_i\colon f(s_i)=i$.
	Let $p_i$ be the probability of the \oneone for leaving $L_i$ for a higher level.
	
	By definition of \Fork{k}{r}, for all $i<k$ there is always at least one $0$ that leads to a higher fitness when flipped, given that all current $1$s remain.
	For these levels $L_0$, \ldots, $L_k$, $i$ equals the number of $1$s in the bit string.
	Recall that $n, k\ge 2$.
	
	\begin{align*}
		\forall i < k: p_i
		\ge \frac{1}{n}\left(1-\frac{1}{n}\right)^i
		\ge \frac{1}{n}\left(1-\frac{1}{n}\right)^n
		\ge \frac{1}{4n}
	\end{align*}
	
	The only fitness levels not covered by this are levels with fitness between $k$ and $k+2$.
	We can leave out $k+2$ since there is no need for leaving the optimum.
	
	Being in level $k$ means having the $1^k$ bit string.
	The only way to leave that level is by either flipping the first $r$ bits to get to level $k+1$ or flipping the last $r$ bits to get to level $k+2$.
	In both cases all other $k-r$ bits have to remain unchanged.
	\begin{align*}
		p_k \ge
		\left(\frac{1}{n}\right)^{r}\left(1-\frac{1}{n}\right)^{k-r}
		\ge \frac{1}{n^r}\left(1-\frac{1}{n}\right)^n
		\ge \frac{1}{4n^r}
	\end{align*}	
	To leave level $k+1$ all $2r$ bits have to be flipped keeping all other bits.
	\begin{align*}
		p_{k+1}
		\ge \left(\frac{1}{n}\right)^{2r}\left(1-\frac{1}{n}\right)^{k-2r}
		\;\ge\; \frac{1}{n^{2r}}\left(1-\frac{1}{n}\right)^n
		\;\ge\; \frac{1}{4n^{2r}}
	\end{align*}
	Using the fitness level argument we get
	\begin{align*}
		E\left(T(k)\right) &\le \sum_{i=0}^{k+1} \frac{1}{p_i}\\
		&= \sum_{i=0}^{k-1} \frac{1}{p_i} + \frac{1}{p_k} + \frac{1}{p_{k+1}}\\
		&\le 4\left(n^r + n^{2r} + \sum_{i=0}^{k-1} n\right)\\
		&= 4\left(n^r + n^{2r} + kn\right)
		\intertext{and having $n \ge k\ge 2r \ge 2$ brings us to}
		&\le 12n^{2r}
		\;\in\; \Or\left(n^{2r}\right).
	\end{align*}
	
	Next, we will to prove that $E\left(T(k)\right) \in \Omega\left(n^{2r}\right)$.
	Let $V$ be the event that the valley occurs within a run.
	From the law of total expectation it follows that
	\begin{align*}
		E\left(T(k)\right) &= E\left(T(k) \; \middle| \; V\right)\Pr{(V)} + E\left(T(k) \; \middle| \; \overline{V}\right)\Pr{(\overline{V})}\\
		&\ge E\left(T(k) \; \middle| \; V\right)\Pr{(V)}\\
		\intertext{Using Lemma \ref{lem:fork_half} and the run time of $n^{2r}$ to escape the valley, one gets}
		&= \frac{1}{2}E\left(T(k) \; \middle| \; V\right)
		\ge \frac{1}{2}n^{2r} \in \Omega\left(n^{2r}\right).
	\end{align*}
\end{proof}

\subsection{Proof of Theorem \ref{theo:single_fork_grouped}}
\begin{theorem*}
	The expected optimization time of the \oneone on \LO with $k$-block \Fork{n}{r} is
	$E(T) \in \Theta\left(\frac{1}{k}n^{2r+1}\right)$.
\end{theorem*}
\begin{proof}
	We use Theorem \ref{theo:lo_exact} to get to this statement.
	As discussed before, to fulfill the requirements, the $1^n$ bit string has to yield the highest fitness.
	However, this can be easily achieved by flipping the last $r$ bits in the definition of \Fork{n}{r}.
	All other properties remain the same since $0$ and $1$ are interchangeable.
	Therefore the run times do not change with this modified definition.
	It holds that $E(T) = \Theta\left(E\left(T^n_k\right)\frac{n}{k}\right)$ where $T^n_k$ is the run time to optimize $f$ on a bit string of length $k$ with a bit flip probability of $\frac{1}{n}$.
	In the previous Theorem \ref{theo:fork_general} we already have shown that
	$T^n_k \in \Theta\left(n^{2r}\right)$.
	Thus, $E(T) \in \Theta\left(\frac{1}{k}n^{2r+1}\right)$.\qed
\end{proof}

\subsection{Proof of Lemma \ref{lem:all_allone}}
\begin{lemma*}
	Let $\lambda$ be the number of islands running the \oneone optimizing \Fork{n}{r}.
	Let $T_{all}^\lambda$ denote the run time for all islands to get to the optimum, valley or $1^n$ as their best solution respectively, regardless of the migration topology and policy.
	If $\lambda$ is polynomial in $n$,
	\begin{align*}
		E\left(T_{all}^\lambda\right) \in \Or\left(n\log n\right).
	\end{align*}
\end{lemma*}
\begin{proof}
	For an upper bound we can assume that we just want to optimize \OM, since the valley and the optimum of \Fork{n}{r} are just exceptions here that make the run time even shorter.
	Further we only consider the worst case that all islands work isolated, since in \OM a higher fitness means being closer to $1^n$.
	
	The expected run time for a single island $E\left(T_{opt}\right)$ until it reaches the optimum, the valley or $1^n$ can be bound from above by the expected number of steps it would take to simply optimize \OM, since the fact that in \Fork{n}{r} optimum and valley can be obtained earlier just decreases the run time.
	Let $p_c$ be the probability for one island to take longer than $cen(\ln\lambda+ \ln n)$ steps to optimize \OM, where $c$ is a constant.
	The multiplicative drift \cite{Doerr2012} gives us the following tail bounds \cite{Doerr2013}.
	\begin{align*}
		\Pr{\left(T_{opt}>en\left(2c\ln\lambda + \ln n\right)\right)} \le e^{-2c\ln\lambda} = \frac{1}{e^c\lambda^2}
	\end{align*}
	Especially for $c\ge1$ we have
	\begin{align*}
		p_c = \Pr{\left(T_{opt}>cen(\ln\lambda+ \ln n)\right)} \le \frac{1}{e^c\lambda^2}.
	\end{align*}
	The run time for all $\lambda$ islands can then be bound by
	\begin{align*}
		p^\lambda_c=\Pr{\left(T_{opt}^\lambda>cen(\ln\lambda+ \ln n)\right)} &= 1-(1-p_c)^\lambda\\
		&\le 1-\left(1-p_c \lambda\right) \le \frac{1}{e^c\lambda}
	\end{align*}
	for $c\ge 1$.
	Let $a_i = ien(\ln\lambda+ \ln n)$.
	Applying the law of total expectation gives us
	\begin{align*}
		E\left(T_{opt}^\lambda\right) &\le \sum_{i=0}^{\infty} E\left(T_{opt}^\lambda \; \middle| \; a_i< T_{opt}^\lambda\le a_{i+1}\right)\Pr{\left(T_{opt}^\lambda>a_i\right)}\\
		&\le \sum_{i=0}^{\infty} a_{i+1}p^\lambda_i\\
		\intertext{and since our bound on $p^\lambda_c$ only holds for $i\ge 1$, }
		&\le a_1 + \sum_{i=1}^{\infty}\frac{a_{i+1}}{e^i\lambda}\\
		&= en(\ln\lambda+ \ln n) + \sum_{i=1}^{\infty}\frac{(i+1)n(\ln\lambda+ \ln n)}{e^i\lambda}\\
		&\le en(\ln\lambda+ \ln n) + \frac{n(\ln\lambda+ \ln n)}{\lambda}\sum_{i=1}^{\infty}\frac{i+1}{e^i}\\
		&\le en(\ln\lambda+ \ln n) + 2n\ln n\sum_{i=1}^{\infty}\frac{i+1}{e^i}\\
		&\le en\ln\lambda+ (e+4)n\ln n\\
		&\in \Or\left(n\log n\right).
	\end{align*}
	
	The last step follows from the restriction of $\lambda$ being polynomial in $n$.
\end{proof}

\subsection{Proof of Theorem \ref{theo:isolated_fork}}

In order to prove the theorem, we first give the following lemma.

\begin{lemma}\label{lem:choose_sum_div}
	For $n\ge 1$,
	\begin{align*}
		\frac{1}{2^n}\sum_{k=1}^{n} \binom{n}{k}\frac{n}{k} \in \Theta\left(1\right).
	\end{align*}
\end{lemma}
\begin{proof}
%

	\begin{align*}
		\frac{1}{2^n}\sum_{k=1}^{n} \binom{n}{k}\frac{n}{k}
		&= \frac{1}{2^n}\sum_{k=1}^{n} \binom{n}{k}\frac{n+1}{k+1}\cdot\frac{n}{n+1}\cdot \frac{k+1}{k}\\
		&= \frac{1}{2^n}\sum_{k=1}^{n} \binom{n+1}{k+1}\frac{n}{n+1}\cdot \frac{k+1}{k}\\
		&= \frac{1}{2^n}\sum_{k=1}^{n} \binom{n+1}{k+1}\Theta\left(1\right)\\
		&= \frac{1}{2^n}\sum_{k=1}^{n-1} \left(\binom{n}{k} + \binom{n}{k+1}\right)\Theta\left(1\right) + \Theta\left(1\right)\\
		&= \frac{1}{2^n}\sum_{k=1}^{n-1} \binom{n}{k}\Theta\left(1\right) + \frac{1}{2^n}\sum_{k=2}^{n} \binom{n}{k}\Theta\left(1\right) + \Theta\left(1\right)\\
		&= \Theta\left(1\right)
	\end{align*}
\end{proof}

\begin{theorem*}
	For $\lambda \le n^r$ isolated islands the expected time to optimize \Fork{n}{r} is
	\begin{align*}
	E\left(T\right) \in \Or\left(n\log(n)+\frac{n^{2r}}{\lambda 2^\lambda}  + \frac{n^r}{\lambda}\right).
	\end{align*}
\end{theorem*}
\begin{proof}
	First assume all islands having the optimum, valley or the $1^n$ string as their best solution respectively.
	From Lemma \ref{lem:all_allone} we know that this is approached in $\Or\left(n\log n\right)$ steps.
	We further assume the worst case, that none of the islands has already found an optimum up to this point.
	As shown in Lemma \ref{lem:fork_half} an island gets into the valley with probability $\frac{1}{2}$.
	If $i$ islands will find the optimum before the valley, the upper bound on the run time to find the optimum is the minimum of
	\begin{enumerate}
		\item the run time it would take at least one of the $i$ islands to make the $n^r$-jump to the optimum
		and
		\item the run time it would take $\lambda-i$ islands to all get to the valley and one of them to get from there to the optimum.
	\end{enumerate}
	It is well known that the expected value for the maximum of $\lambda$ independently sampled elements of the same geometric distribution with success probability of $p$ can be bound from above by $\frac{1}{p}\sum_{k=1}^{\lambda}\frac{1}{k} \le \frac{\log\lambda+1}{p}$ \cite{Eisenberg2008}.
	We can adapt this to our upper bound on the run time for $\lambda-i$ islands to get to the valley from $1^n$, since this is also geometrically distributed for each island.
	Further, when we observe the jumps to the valley or the optimum respectively, we make use of Equation \eqref{eq:all_jump}.
	From the law of total expectation we can observe that 
	
	\begin{align*}
		E\left(T\right)&\le \Or\left(n\log(n)\right)\\
		&+ \sum_{i=0}^{\lambda}\left(\frac{1}{2}\right)^\lambda \binom{\lambda}{i} \min\left(\frac{n^r}{i}+1, 2n^r\log\lambda + \frac{n^{2r}}{\lambda-i}+1\right)\\
		&= \Or\left(n\log(n)\right)\\
		&+\sum_{i=0}^{\lambda}\left(\frac{1}{2}\right)^\lambda \binom{\lambda}{i} \min\left(\frac{n^r}{i}, 2n^r\log\lambda + \frac{n^{2r}}{\lambda-i}\right).
	\end{align*}
	To eliminate the minimum in the inequality, we want to find out from which $i$ on $\frac{n^r}{i} \le 2n^r\log\lambda + \frac{n^{2r}}{\lambda-i}$.
	Obviously this $i$ equals 1 since zero islands cannot make progress at all, but for any larger $i$, the first term is always smaller than the second one.
	Knowing this, one can resolve the minimum to get to
	\begin{align*}
		E\left(T\right) &\le \Or\left(n\log(n)\right)+\frac{2n^r\log\lambda}{2^\lambda} + \frac{n^{2r}}{\lambda2^\lambda} + \sum_{i=1}^{\lambda}\left(\frac{1}{2}\right)^\lambda \binom{\lambda}{i} \frac{n^r}{i}.
	\end{align*}
	The second term on the right side is always smaller than the first or the third one:
	For $\lambda \ge 2r\log n$, $\frac{2n^r\log\lambda}{2^\lambda} < n\log n$ and for $\lambda \le 2r\log n$, $\frac{2n^r\log\lambda}{2^\lambda} < \frac{n^{2r}}{\lambda2^\lambda}$.
	Therefore we make the right side of the inequality just greater by doubling the outer two and leaving out the middle one.
	\begin{align*}
		E\left(T\right) &\le \Or\left(2n\log(n)\right) + \frac{2n^{2r}}{\lambda2^\lambda} + \frac{n^r}{2^\lambda}\sum_{i=1}^{\lambda} \binom{\lambda}{i}\frac{1}{i}
	\end{align*}
	Using Lemma \ref{lem:choose_sum_div} gives
	\begin{align*}
		E\left(T\right) \in \Or\left(n\log(n) + \frac{n^{2r}}{\lambda2^\lambda}  + \frac{n^r}{\lambda}\right).
	\end{align*}
\end{proof}

\subsection{Proof of Lemma \ref{lem:worst_case_runtime}}
\begin{lemma*}
	For $\lambda \in \Or\left(\frac{n^{2r-1}}{\log n}\right)$ islands and any migration topology and/or policy, the expected run time to optimize \Fork{n}{r} can be bounded by
	\begin{align*}
		E(T) \in \Or\left(\frac{n^{2r}}{\lambda}\right).
	\end{align*}
\end{lemma*}
\begin{proof}
	From Lemma \ref{lem:all_allone} we know that the expected number of steps until all islands have found either the valley, $1^n$ or the optimum is in $\Or\left(n\log n\right)$.
	If no island already found the optimum after that time, every island has a chance of at least $\frac{n^{2r}}{\lambda}$ to get to the optimum in the next step.
	
	By using Equation \eqref{eq:all_jump}, the expected run time for one out of $\lambda$ islands to make the final jump to the optimum can be bounded by $\Or\left(\frac{n^{2r}}{\lambda}\right)$.\\
	Therefore we get an overall run time of $\Or\left(n\log n + \frac{n^{2r}}{\lambda}\right) = \Or\left(\frac{n^{2r}}{\lambda}\right)$.
\end{proof}

\subsection{Proof of Lemma \ref{lem:ov_bound}}
\begin{lemma*}
	The probability $p_{ov}$ for a single island to generate the optimum or the valley during its way to $1^n$ while optimizing \Fork{k}{r} with \oneone$_{1/n}$ is 
	\begin{align*}
		p_{ov} \in \Or\left(\frac{1}{k^{r-1}}\right)\;\;\text{ for }k\le n.
	\end{align*}
\end{lemma*}
\begin{proof}
	We start by investigating the expected number of times $E(J)$ that the \oneone generates an individual with $r$ zeros and $k-r$ ones, even if this individual will be discarded.
	Recall that also the optimum and the valley have that number of zeros and ones.
	Let $S_i$ be the number of jumps that the \oneone makes from fitness level $k-i$ to the desired level of exactly $r$ zeros.
	It follows that,
	\begin{align*}
		E(J) &\le E(1 + S_r + S_{r-1}+\dots + S_1)\\
		&= 1+\sum_{i=1}^r E(S_i).
	\end{align*}
	The ``$+1$'' comes into play when the level with $r$ zeros is reached from a lower fitness level, which can only happen once by the definition of the \oneone.
	The next step now is to resolve $E(S_i)$.
	
	Let $T_i$ denote the number of steps the \oneone, while optimizing \OM, stays on fitness $k-i$.
	By having a geometric distribution, $T_i$ equals the reciprocal of the probability $p_i$ that we leave the fitness level $k-i$ for a higher one.
	This is at least he case, if one arbitrary zero is flipped to one while all other bits remain the same.
	It follows that
	\begin{align*}
		p_i &\ge \binom{i}{1}\frac{1}{n}\left(1-\frac{1}{n}\right)^{k-1}
		\quad=\quad \frac{i}{n}\left(1-\frac{1}{n}\right)^{n-1}
		\quad\ge\quad \frac{1}{en}
	\end{align*}
	and hence, $E(T_i) \le en$.
	
	Further let $X_{i, j}$ be the indicator variable that equals $1$ if a bit string with exactly $r$ zeros is generated on mutation step $j$ while the \oneone is at fitness $k-i$, else $0$.
	We can observe that $S_i = \sum_{j=1}^{T_i} X_{i, j}$.
	Using Wald's equation we get $E\left(S_i\right) = E(T_i)E(X_{i, 1})$.
	Recall that for Wald's equation $X_{i,j}$ has to have the same distribution for every $j$ and a fix arbitrary $i$.
	Hence,
	\begin{align}
		E(J) &\le 1+\sum_{i=1}^r E(S_i) \quad=\quad 1+\sum_{i=1}^r E(T_i)E(X_{i, 1})\nonumber\\
		&= 1+\sum_{i=1}^{r-1} E(T_i)E(X_{i, 1})+E(T_r)E(X_{r, 1}) \label{eq:wald_num_of_VO}.
	\end{align}
	Next we will resolve $E(X_{i, 1})$.
	In the sum, $i=r$ is a special case, because this level already is the level we want to get to and at least two bits have to be flipped to get another bit string on the same level.
	To distinguish, we investigate two cases.
	\begin{itemize}
		\item Case 1: $i\le r-1$
		
		Let $Z$ be the number of zeros flipped to one in a step.
		The more zeros flip, the more have to be flipped back, since we already need more zeros.
		\begin{align*}
			E(X_{i,1}) &= \sum_{h=0}^\infty E\left(X_{i,1} \; \middle| \; Z=h\right)\Pr{(Z=h)}\\
			&\le \sum_{h=0}^\infty E\left(X_{i,1} \; \middle| \; Z=0\right)\Pr{(Z=h)}\\
			&= E\left(X_{i,1} \; \middle| \; Z=0\right)\sum_{h=0}^\infty \Pr{(Z=h)}\\
			&= E\left(X_{i,1} \; \middle| \; Z=0\right)\\
			&= \left(\frac{1}{n}\right)^{r-i} \binom{k-i}{r-i} \left(1-\frac{1}{n}\right)^{k-r}\\
			\intertext{Knowing that $r<k$ by definition of \Fork{k}{r} leads to}
			&\le \left(\frac{1}{n}\right)^{r-i} \binom{k}{r-i}\\
			&\le \left(\frac{k}{n}\right)^{r-i}
			\quad\le\quad \frac{k}{n}.
		\end{align*}
		
		\item Case 2: $i=r$:
		
		We make a sum over all possible numbers of zeros that flip to one.
		The same amount of ones has to flip to zero.
		\begin{align*}
			E(X_{r,1}) &= \sum_{h=1}^{r} \left(\frac{1}{n}\right)^{2h}\binom{r}{h}\binom{n-r}{h}\\
			&\le \sum_{h=1}^{r} \left(\frac{1}{n}\right)^2\binom{r}{1}\binom{n-r}{1}\\
			&= r^2(n-r)\left(\frac{1}{n}\right)^2\\
			&\le \frac{r^2}{n}
		\end{align*}
	\end{itemize}
	If we now put $E(T_i)$ and $E(X_{i,1})$ into Equation \eqref{eq:wald_num_of_VO}, we get
	\begin{align*}
		E(J) &\le 1+\sum_{i=1}^{r-1} \left(en \frac{k}{n}\right)+en\frac{r^2}{n}\\
		&\le 1 + erk + er^2.
	\end{align*}
	
	Since this is an upper bound on the expected number of bit strings with exactly $r$ zeros that are seen during one run of the \oneone, we can get an upper bound on the probability that we see the valley ($p_v$) or the optimum ($p_o$), respectively, by
	\begin{align*}
		p_o = p_v &\le \frac{1 + erk + er^2}{\binom{k}{r}}.
	\end{align*}
	Since $r$ is constant, there is a constant $c$ so that the following holds.
	\begin{align*}
		p_o = p_v &\le \frac{1 + erk + er^2}{ck^r}\\
		&= \frac{1}{ck^r} + \frac{erk}{ck^r} + \frac{er^2}{ck^r}\\
		&\le \frac{d}{k^{r-1}}\quad\text{for a constant $d>0$}
	\end{align*}
	Applying the union bound we finally get
	\begin{align*}
		p_{ov} &\le p_o + p_v = \frac{2d}{k^{r-1}}\\
		&\in \Or\left(\frac{1}{k^{r-1}}\right).
	\end{align*}
\end{proof}

\subsection{Proof of Lemma \ref{lem:probability_Q}}
\begin{lemma*}
	Under the condition of at least one of $\lambda \in \Or\left(n^{r-1}\right)$ island having $1^n$ and all others the valley as solution, the probability for $Q$ that one island will find the valley and a migration will be made before the optimum is found by any island is
	\begin{align*}
		\frac{c}{2} \frac{n^r}{n^r + \tau\lambda}
		\;\le\; \Pr{(Q)} \;\le\;
		\frac{n^r + \frac{\lambda}{2}}{2n^r + \tau\lambda}.
	\end{align*}
	for a constant $0<c<1$.
\end{lemma*}
\begin{proof}
	Starting with the examination of the lower bound, we look at the island $\Lambda$ that comes up with the valley or the optimum at first.
	With probability $\frac{1}{2}$ it finds the valley before the optimum (Lemma~\ref{lem:fork_half}).
	Until the next migration happens, no island should find the optimum to make $Q$ happen.
	We already know from Lemma~\ref{lem:ov_bound} that $\left(1-\frac{1}{n^r}\right)^\lambda$ is a lower bound on the probability of not finding the optimum by at least one of $\lambda$ islands in the next step.
	This has to hold for all steps until a migration event occurs.
	In the following sum, $i$ represents the index of the step where $\Lambda$ shares the valley to all other islands.
	Thus,
	\begin{align*}
		\Pr{(Q)}&= \frac{1}{2} \sum_{i=1}^{\infty} \left(1-\frac{1}{\tau}\right)^{i-1} \frac{1}{\tau}\left(1-\frac{1}{n^r}\right)^{\lambda i}\\
		&= \frac{\left(1-\frac{1}{n^r}\right)^\lambda}{2\tau} \sum_{i=0}^{\infty} \left(\left(1-\frac{1}{\tau}\right)\left(1-\frac{1}{n^r}\right)^\lambda\right)^i\\
		&\ge \frac{\left(1-\frac{1}{n^r}\right)^\lambda}{2\tau} \sum_{i=0}^{\infty} \left(\left(1-\frac{1}{\tau}\right)\left(1-\frac{\lambda}{n^r}\right)\right)^i.
		\intertext{By knowing that $\lambda \in \Or\left(n^{r-1}\right)$, there is a constant $0<c<1$ so that $c \le \left(1-\frac{1}{n^r}\right)^\lambda$.}
		&\ge \frac{c}{2\tau} \sum_{i=0}^{\infty} \left(\left(1-\frac{1}{\tau}\right)\left(1-\frac{\lambda}{n^r}\right)\right)^i\\
		&= \frac{c}{2\tau} \frac{1}{1-\left(1-\frac{1}{\tau}\right)\left(1-\frac{\lambda}{n^r}\right)}\\
		&= \frac{c}{2\tau} \frac{1}{\frac{1}{\tau} + \frac{\lambda}{n^r}-\frac{\lambda}{\tau n^r}}\\
		&\ge \frac{c}{2} \frac{1}{1 + \frac{\tau\lambda}{n^r}}\\
		&= \frac{c}{2} \frac{n^r}{n^r + \tau\lambda}
	\end{align*}
	which proves the lower bound.
	
	For the counterpart we assume all islands that come up with the valley during the whole run time already start with the valley.
	This also means that the next new individual found by any island is the optimum.
	
	The probability that the number of islands that get the valley is at least $\frac{\lambda}{2} \ge r\log n$ is at most $\left(\frac{1}{2}\right)^{r\log n} = \frac{1}{n^r}$, since $\frac{1}{2}$ is the probability to get the valley instead of the optimum (Lemma \ref{lem:fork_half}).
	By using Lemma \ref{lem:worst_case_runtime} and the law of total expectation we see that this case can be ignored, as the resulting term will not dominate the bound.
	Thus, we will now assume that we start with at most $\frac{\lambda}{2}$ valleys and the next change on an island will be the optimum.
	
	Since all islands migrate always at the same time, we get an upper bound on $\Pr{(Q)}$ similarly to the lower bound.
	\begin{align*}
		\Pr{(Q)}&= \frac{1}{2} \sum_{i=1}^{\infty} \left(1-\frac{1}{\tau}\right)^{i-1} \frac{1}{\tau}\left(1-\frac{1}{n^r}\right)^{\left(\lambda-\frac{\lambda}{2}\right) i}\\
		&\le \frac{1}{2} \sum_{i=1}^{\infty} \left(1-\frac{1}{\tau}\right)^{i-1} \frac{1}{\tau}\left(1-\frac{1}{n^r}\right)^{\frac{\lambda i}{2}}\\
		&= \frac{\left(1-\frac{1}{n^r}\right)^{\frac{\lambda}{2}}}{2\tau} \sum_{i=0}^{\infty} \left(\left(1-\frac{1}{\tau}\right)\left(1-\frac{1}{n^r}\right)^{\frac{\lambda}{2}}\right)^i\\
		&\le \left(\frac{1}{2\tau}\right) \frac{1}{1-\left(1-\frac{1}{\tau}\right)\left(1-\frac{1}{n^r}\right)^{\frac{\lambda}{2}}}\\
		&\le \left(\frac{1}{2\tau}\right) \frac{1}{1-\left(\frac{\tau-1}{\tau}\right)\left(\frac{n^r}{n^r+\frac{\lambda}{2}}\right)}\\
		&= \left(\frac{1}{2\tau}\right) \frac{1}{1-\left(\frac{\tau n^r - n^r}{\tau n^r + \frac{\tau\lambda}{2}}\right)}\\
		&= \frac{n^r + \frac{\lambda}{2}}{2n^r + \tau\lambda}.
	\end{align*}
\end{proof}

\subsection{Proof of Theorem \ref{theo:fork_comlete_lower_bound}}
\begin{theorem*}
	The expected run time of $2r\log n \le \lambda \in \Or\left(\frac{n^{r-1}}{\log n}\right)$ islands optimizing \Fork{n}{r} on a complete graph is $E(T) \in \Omega\left(\frac{n^{3r} + \tau\lambda n^r}{\lambda n^r + \tau\lambda^2}\right)$.
\end{theorem*}
\begin{proof}
	To get a lower bound, for this proof we want to assume starting with all islands having $1^n$ or the valley as best solution.
	Let the event that this state is ever reached be called $U$.
	To show that our assumption leads to a lower bound we now have to prove that $\Pr{(U)}$ is at least constant.
	To let $U$ happen, no island is allowed to find the optimum until $U$ is satisfied.
	If we look at Lemma~\ref{lem:ov_bound}, we know that the probability for one island to generate the optimum during this time is at most $\frac{b}{n^{r-1}}$ for a constant $b>0$ .
	We can conclude that $\Pr{(U)} \ge \left(1-\frac{b}{n^{r-1}}\right)^\lambda$.
	From the limitation $\lambda \in \Or\left(n^{r-1}\right)$ we can derive $\Pr{(U)} \ge d$ for a constant $0<d\le1$.
	This confirms that we can derive the lower bound under the assumption we wanted to make. 
	
	For the event $Q$ like defined in Lemma~\ref{lem:probability_Q}, we use that $E(T) = E(T \;|\; Q)\Pr{(Q)} + E(T \;|\; \overline{Q})\Pr{(\overline{Q})}$ and will show the values of both summands in the following.
	Also from that Lemma we receive the bounds on $\Pr{(Q)}$.
	
	If $Q$ happens, meaning all islands get the valley as solution, all islands will have to make the jump to the optimum from the valley.
	The expected overall run time under this condition therefore is at least $\frac{n^{2r}}{2\lambda}$ (Equation \eqref{eq:all_jump}).
	Therefore, for this event we get a total of
	$
	E\left(T \; \middle|  \; Q\right)\Pr{(Q)} \ge \frac{n^{2r}}{2\lambda}\frac{dc}{2} \left(\frac{n^r}{n^r + \tau\lambda}\right)
	\;\in\; \Omega\left(\frac{n^{3r}}{\lambda n^r + \tau\lambda^2}\right)
	$.
	
	We now continue with the other part, $E\left(T \; \middle|  \; \overline{Q}\right)\Pr{(\overline{Q})}$.
	The expected run time until one of $\lambda$ islands comes up with the optimum is $E{(T \;|\;\overline{Q})} \in \Omega\left(\frac{n^r}{\lambda}\right)$ (Equation~\eqref{eq:all_jump}) which leads us to
	$
	E{(T \;|\;\overline{Q})}\Pr{(\overline{Q})} \in \Omega\left(\frac{\tau n^r}{n^r + \tau\lambda}\right)
	$.
	
	If we finally add both cases of $Q$ and $\overline{Q}$ we get
	\begin{align*}
		E(T) &= E{(T \;|\; Q)}\Pr{(Q)} + E{(T \;|\;\overline{Q})}\Pr{(\overline{Q})}
		\;\in\; \Omega\left(\frac{n^{3r} + \tau\lambda n^r}{\lambda n^r + \tau\lambda^2}\right).
	\end{align*}\qed
\end{proof}

\subsection{Proof of Theorem \ref{theo:fork_comlete_upper_bound}}
\begin{theorem*}
	The expected run time of $2\log n \le \lambda \in \Or\left(\frac{n^{r-1}}{\log n}\right)$ islands optimizing \Fork{n}{r} on a complete graph is in
	$
	E(T) \in \Or\left(\frac{n^{3r} + \tau\lambda n^r}{\lambda n^r + \tau\lambda^2} + \frac{n^{2r+1}\log n}{\tau\lambda}\right)
	$.
\end{theorem*}
\begin{proof}
	To derive an upper bound we assume for the whole proof that until all islands have $1^n$ or the valley as best individual, no optimum is generated.
	Similar to the proof for the lower bound we want to split the run time by the event $U$, that all islands come to a state where at least one island has $1^n$ and all others the valley as their best individual.
	Unlike before, we cannot ignore the case of $\overline{U}$, since we want an upper bound.
	We start with $\overline{U}$.
	The only two reasons for $\overline{U}$ to happen are that
	\begin{enumerate}[nosep]
		\item all islands create the valley on their own or
		\item at least one island finds the valley during these $dn\log n$ steps and migrates it to all others
	\end{enumerate}
	
	(1): The probability for this is not more than $\left(\frac{1}{2}\right)^{2r\log n} = \frac{1}{n^{2r}}$, because we have at least $2r\log n$ islands.
	As in Theorem \ref{theo:fork_comlete_lower_bound} we can ignore that case because of Lemma \ref{lem:worst_case_runtime} and the law of total expectation.
	
	(2): To calculate the probability for that, we assume the worst case that there are already islands with the valley from the beginning on.
	We know that after $\Or\left(n\log n\right)$ steps, the decision of $U$ or $\overline{U}$ has been made (Lemma \ref{lem:all_allone}).
	The probability to migrate during that time is for a constant $d>0$
	\begin{align*}
		\Pr{(\overline{U})} &\le \sum_{i=0}^{dn\log n} \left(1-\frac{1}{\tau}\right)^i\frac{1}{\tau}
		\;\le\; \frac{2dn\log n}{\tau + dn\log n}\;\le\; \frac{2dn\log n}{\tau}.
	\end{align*}
	The last step could be made by the inequality introduced by Badkobeh et al. \cite{Badkobeh2015}.
	The run time in that case would be in $\Or\left(\frac{n^{2r}}{\lambda}\right)$, using Lemma~\ref{lem:worst_case_runtime}.
	It follows that
	$
	E{(T\;|\;\overline{U})}\Pr{(\overline{U})} \in \Or\left(\frac{n^{2r+1}\log n}{\tau\lambda}\right)
	$.
	
	It is still left to show the run time under the condition of $U$.
	Therefore for the rest of the proof we will assume that everything happens under the condition of $U$.
	Like for the lower bound we split up the run time again into two cases of $Q$ and $\overline{Q}$ defined like in Lemma~\ref{lem:probability_Q}.
	Starting with $Q$, we get an expected run time by using Lemma \ref{lem:worst_case_runtime} which results in $E(T \;|\; Q) \in \Or\left(\frac{n^{2r}}{\lambda}\right)$.
	We already showed the bounds on $\Pr{(Q)}$ in Lemma~\ref{lem:probability_Q}, especially $\Pr{(Q)} \le \frac{n^r+ \frac{\lambda}{2}}{2n^r + \tau\lambda}$.
	Putting both together leads to 
	$
	E(T \;|\; Q)\Pr{(Q)} \in \Or\left(\frac{n^{3r}}{\lambda n^r + \tau\lambda^2}\right)
	$.
	
	In the case of $\overline{Q}$, the run time it takes these island to get to the optimum is in $\Or\left(n\log n + \frac{n^r}{\lambda}\right)$.
	This follows from Equation~\eqref{eq:all_jump} and the fact that we have to make $\Or\left(n\log n\right)$ steps to get to make $U$ happen in the first place (Lemma \ref{lem:all_allone}).
	Again by looking at Lemma~\ref{lem:probability_Q}, we see that $\Pr{(\overline{Q})} \le \frac{n^r(2-c) + 2\tau\lambda}{2n^r + 2\tau\lambda} \le \frac{n^r(2-c) + 2\tau\lambda}{2n^r + 2\tau\lambda}$.
	If $\tau\lambda \le n^r$, we can observe that $\Pr{(\overline{Q})} \in \Or(\Pr{(Q)})$ and hence $E(T \;|\; \overline{Q})\Pr{(\overline{Q})} \in \Or\left(E(T \;|\; Q)\Pr{(Q)}\right)$.
	In this case the bound we already derived for that is enough.
	The other case would be that $n^r < \tau\lambda$.
	Therefore, $\Pr{(\overline{Q})} \le \frac{\tau\lambda(4-c)}{2n^r + 2\tau\lambda}$ and therefore
	$
	E(T \;|\; \overline{Q})\Pr{(\overline{Q})} \in \Or\left(\frac{\tau n^r + \tau\lambda n\log n}{n^r + \tau\lambda}\right)
	\;\in\; \Or\left(\frac{\tau n^r}{n^r + \tau\lambda}\right)
	$.
	The last step could be made because of $\lambda$ being in $\Or\left(\frac{n^{r-1}}{\log n}\right)$.
	Finally we add up the three results together to obtain
	\begin{align*}
		E(T) &= E{(T\;|\;\overline{U})}\Pr{(\overline{U})} + E(T \;|\; Q)\Pr{(Q)} + E(T \;|\; \overline{Q})\Pr{(\overline{Q})}\\
		&\in \Or\left(\frac{n^{3r} + \tau\lambda n^r}{\lambda n^r + \tau\lambda^2} + \frac{n^{2r+1}\log n}{\tau\lambda}\right).
	\end{align*}
	Recall that this holds because $Q$ and $\overline{Q}$ were made under the condition of $U$.\qed
\end{proof}

\subsection{Proof of Lemma \ref{lem:ring_constant_valleys_until_allone}}
\begin{lemma*}
	If $b\ge 1+\frac{r}{\epsilon}$ is constant and $\lambda \in \Or\left(n^{r-1-\epsilon}\right)$, where $\epsilon>0$ is a constant, it holds that the expected run time of to optimize \Fork{n}{r} on any island is
	$
	E(T) \le E{\left(T \;\middle|\; V_b\right)} + \Or\left(n\right)
	$.
\end{lemma*}
\begin{proof}
	We know that $E(T) = E{(T \;|\; \overline{V_b})}\Pr{(\overline{V_b})}+E{\left(T \;\middle|\; V_b\right)}\Pr{\left(V_b\right)}$, so what is left to prove is that $E{(T \;|\; \overline{V_b})}\Pr{(\overline{V_b})} \in \Or(n)$.
	
	We use Lemma \ref{lem:ov_bound} to derive that $\Pr{(\overline{V_b})} \le \left(\frac{1}{n^{r-1}}\right)^b\binom{\lambda}{b} \le \left(\frac{\lambda}{n^{r-1}}\right)^b$.
	For a constant $m>0$, the expected worst case run time when having too many valleys at most $\frac{mn^{2r}}{\lambda}$ (Lemma \ref{lem:worst_case_runtime}).
	Hence, for a constant $d>0$ so that $\lambda \le dn^{r-1-\epsilon}$,
	\begin{align*}
		E{\left(T \;\middle|\; \overline{V_b}\right)}\Pr{(\overline{V_b})}
		& \leq \frac{mn^{2r}}{\lambda}\left(\frac{\lambda}{n^{r-1}}\right)^b
		\;=\; mn^{2r}\frac{\lambda^{b-1}}{n^{br-b}}
		\;\le\; md^{b-1} n^{r+1+\epsilon-b\epsilon}
	\end{align*}
	By having $b\ge 1+\frac{r}{\epsilon}$ we get to $E{\left(T \;\middle|\; \overline{V_b}\right)}\Pr{(\overline{V_b})}\le md^{b-1} n\;\in\; \Or\left(n\right)$.\qed
\end{proof}

\subsection{Proof of Lemma \ref{lem:ring_constant_migration}}
\begin{lemma*}
	Let $\epsilon>0$, $b\ge 1+\frac{r}{\epsilon}$ and $c\ge 7rb$ be constants and $\lambda \in \Or\left(n^{r-1-\epsilon}\right)$.
	When $\frac{1}{\tau}$ denotes the migration probability and $\tau \in \Omega\left(n\log n\right)$, the expected run time to optimize \Fork{n}{r} on any island is
	$
	E(T) \le E{\left(T \;\middle|\; V_b\cap B_c\right)} + \Or\left(\frac{n^r}{\lambda}\right).
	$
\end{lemma*}
\begin{proof}
	We already know from Lemma \ref{lem:ring_constant_valleys_until_allone} that $E(T) \le E{\left(T \;\middle|\; V_b\right)} + \Or\left(n\right)$.
	From now on we assume for simplicity that all probabilities and expected values are under the condition of $V_b$.
	If we can prove that $E{(T \;|\; \overline{B_c})}\Pr{(\overline{B_c})} \in \Or\left(\frac{n^r}{\lambda}\right)$, we are done since $\Or(n) \subseteq \Or\left(\frac{n^r}{\lambda}\right)$.
	We assume $E{(T \;|\; \overline{B_c})}$ to have the worst run time of $\Or\left(\frac{n^{2r}}{\lambda}\right)$ steps (Lemma \ref{lem:worst_case_runtime}).
	Thus, there is a constant $m>0$ so that this worst case run time is $E{(T \;|\; \overline{B_c})} \le \frac{mn^{2r}}{\lambda}$.
	
	Next we examine $\Pr{(\overline{B_c})}$.
	We first look at one island that came up with the valley on its own and sends this individual to its neighboring islands.
	If it migrates at least once, two more islands adopt the valley.
	From there on, to let one more island know the solution, one of the two islands has to migrate.
	Then this new island becomes the one that has to migrate and so on.
	This way there will be a cluster of adjacent islands on the Ring where each has the valley as optimal solution.
	With probability $\frac{1}{\tau}$ a migration is made.
	Because there are only two outer most bits in a cluster, the number of islands reached in one step is $2$.
	We will use $d>0$ as a constant to express that the run time for all islands to get $1^n$, the optimum or the valley is at most $dn\log n$ (see Lemma~\ref{lem:all_allone}).
	The expected amount of migrations made during that time is $\frac{dn\log n}{\tau}$.
	After the $dn\log n$ steps we have exactly $c\log n$ valleys, if the $b$ valleys migrate $\frac{c\log n-b}{2b}$ times.
	The probability to spread the found $b$ valleys at least $S\ge\frac{c\log n-b}{2b}$ times is 
	\begin{align*}
		\Pr\left(\overline{B_c}\right)&=\Pr\left(S \ge \frac{c\log n-b}{2b}\right)
		\;=\; \Pr\left(S \ge \frac{dn\log n}{\tau}+\frac{c\log n-b}{2b}-\frac{dn\log n}{\tau}\right)\\
		&= \Pr\left(S \ge \left(1+\frac{c\tau\log n-\tau b}{2bdn\log n}-1\right)\frac{dn\log n}{\tau}\right)\\
		&< e^{-\frac{\left(\frac{c\tau\log n-\tau b}{2bdn\log n}-1\right)\left(\frac{dn\log n}{\tau}\right)}{3}}
		\;=\; e^{-\frac{c\log n-b}{6b}+\frac{dn\log n}{3\tau}}\\
		&= e^{-\frac{c\log n}{6b}+\frac{1}{6}+\frac{dn\log n}{3\tau}}.
	\end{align*}
	Since $\tau \in \Omega\left(n\log n\right)$, for $n$ large enough we get
	$\Pr\left(\overline{B_c}\right) \le e^{-\frac{c\log n}{7b}} \;=\;  \frac{1}{n^{\frac{c}{7b}}} \;\le\; \frac{1}{n^r}$ because of $c\ge 7rb$ and therefore,
	$
	E{(T \;|\; \overline{B_c})}\Pr{(\overline{B_c})}
	\le \frac{mn^{2r}}{\lambda}\left(\frac{1}{n^r}\right)
	\;=\; \Or\left(\frac{n^r}{\lambda}\right)
	$.\qed
\end{proof}

\subsection{Proof of Theorem \ref{theo:ring_fork}}
\begin{theorem*}
	I we have $\tau \in \Omega(n\log n)$, $12r\log n \le \lambda \in \Or\left(n^{r-1-\epsilon}\right)$ and $\lambda^2 \tau \ge 17r^2n^r\log^2 n$, then the expected optimization time for \Fork{n}{r} on a Ring topology is
	$
	E(T) \in \Or\left(\frac{n^r}{\lambda}\right).
	$
\end{theorem*}
\begin{proof}
	From Lemma \ref{lem:ring_constant_migration} we know that $E(T) \le E{\left(T \;\middle|\; V_b\cap B_c\right)} + \Or\left(\frac{n^r}{\lambda}\right)$, where $b>0$ and $c\ge 7br$ are constant.
	What is left to show is that $E{\left(T \;\middle|\; V_b\cap B_c\right)} \in \Or\left(\frac{n^r}{\lambda}\right)$.
	Therefore we can prove the theorem by assuming that we start with $c\log n$ islands having the valley and the rest having $1^n$ as their solution.
	Of course we have to add $\Or\left(n\log n\right)$ to the bound in the end, which is the expected time to get to that state (Lemma~\ref{lem:all_allone}).\\
	\indent We get an upper bound if we assume the worst case that all islands that would generate the valley from now on before the optimum is found already start with the valley, while all others have $1^n$ as solution so far.
	The probability that this number of islands is at least $r\log n$ is $\left(\frac{1}{2}\right)^{r\log n} = \frac{1}{n^r}$.
	From the law of total expectation and Lemma \ref{lem:worst_case_runtime} it follows that we can ignore that case since it does not exceed the bounds we want to prove.\\
	\indent Due to the mentioned reason we consider that already $8r\log n$ islands have the valley as their best solution from the beginning on.
	We will show now that the time that is left for the other islands until they adopt the valley by migration is enough so that at least one island finds the optimum.
	We get an upper bound if we assume that the $8r\log n$ islands with the valley are distributed evenly such that there are $8r\log n$ many of these clusters of non-valley islands.\\
	\indent We want to examine now how many islands adopt the valley after $\frac{2rn^r\log n}{\lambda}$ steps.
	Since we expect $2r\log n$ islands to adopt the valley each $\tau$ steps, the expected value of valley islands after this time is $\frac{4r^2n^r\log^2 n}{\tau\lambda} + 8r\log n \le \frac{\lambda}{4}$ for $n$ large enough.
	To get this we make use of $\tau\lambda^2 \ge 17r^2n^r\log^2 n$.
	Using Chernoff bounds, the probability to lose more than the double this number is at most
	$\Pr{\left(X \ge (1+1)\frac{\lambda}{4}\right)} < \frac{1}{e^{\frac{\lambda}{12}}} \le \frac{1}{e^{\frac{12r\log n}{12}}} = \frac{1}{n^r}$.
	Again by using the law of total expectation and Lemma \ref{lem:worst_case_runtime} we can ignore that case.
	Therefore we know that, after $\frac{2rn^r\log n}{\lambda}$ steps, we still have at least $\frac{\lambda}{2}$ islands left.
	Hence, the number of evaluations that were made during that time is at least $rn^r\log n$.
	The probability to not find the optimum during that time is at most
	$
	p \le \left(1-\frac{1}{n^r}\right)^{rn^r\log n} \le \left(1-\frac{1}{n^r}\right)^{rn^r\ln n} \le \frac{1}{n^r}.
	$	
	We now use that
	\begin{align*}
	E(T) &= E\left(T \;\middle|\; T<\frac{2rn^r\log n}{\lambda}\right)\Pr{\left(T<\frac{2rn^r\log n}{\lambda}\right)}\\
	&+E\left(T \;\middle|\; T\ge\frac{2rn^r\log n}{\lambda}\right)\Pr{\left(T\ge\frac{2rn^r\log n}{\lambda}\right)}\\
	&\le E\left(T \;\middle|\; T<\frac{2rn^r\log n}{\lambda}\right)\cdot 1 \;+\; \frac{n^{2r}}{\lambda }\cdot \frac{1}{n^r}.
	\end{align*}
	We already know that the number of islands during $T<\frac{2rn^r\log n}{\lambda}$ steps is at least $\frac{\lambda}{2}$.
	Hence,
	$
	E(T) \le E\left(T \;\middle|\; T<\frac{2rn^r\log n}{\lambda}\right) + \frac{n^r}{\lambda}
	\;\le\; \frac{2n^r}{\lambda} + \frac{n^{r}}{\lambda}
	\;=\; \frac{3n^r}{\lambda} \in \Or\left(\frac{n^r}{\lambda}\right).
	$
	If we finally add $\Or\left(n\log n\right)$ for the run time until they all found $1^n$ or the valley in the first place, we get $\Or\left(n\log n + \frac{n^r}{\lambda}\right)$.
	This also matches the bound we still have to add because of Lemma \ref{lem:ring_constant_migration}.
	The term $n\log n$ is dominated by $\frac{n^r}{\lambda}$, because of the restriction to $\lambda$.
	Which leads to an upper bound of $\Or\left(\frac{n^r}{\lambda}\right)$.\qed
\end{proof}

\end{document}